\documentclass[onecolumn]{IEEEtran}

\usepackage{amssymb}
\usepackage{amsmath}
\usepackage{graphicx}
\usepackage{pgf,tikz}
\usetikzlibrary{arrows}

\usepackage{algorithmic}
\usepackage{algorithm}
\usepackage{hyperref}

\newtheorem{lemma}{Lemma}
\newtheorem{theorem}{Theorem}
\newtheorem{definition}{Definition}
\newtheorem{example}{Example}

\newcommand{\opt}{\mathrm{opt}}
\newcommand{\non}{\mathrm{non}}

\begin{document}

\title{Pure Strategy or Mixed Strategy? }
\author{Jun He,
\thanks{Jun He   is with Department of Computer Science, Aberystwyth University, Ceredigion, SY23 3DB, UK. Email:  jun.he@aber.ac.uk.}
  \and Feidun He, 
  \thanks{Feidun He is with School of Information Science and Technology, Southwest Jiaotong University, Chengdu, Sichuan, 610031, China}
    \and Hongbin Dong 
    \thanks{Hongbin Dong is with College of Computer Science and Technology, Harbin Engineering University, Harbin, 150001, China}
    }

\maketitle

\begin{abstract}
Mixed strategy  evolutionary algorithms (EAs) aim at integrating several mutation operators into a single algorithm.
However no analysis has been made to answer the theoretical question: whether  and when is the performance of mixed strategy EAs  better than that of pure strategy EAs?  In this paper,  asymptotic convergence rate and asymptotic hitting time  are proposed to measure the performance of EAs.
It is proven  that  the asymptotic convergence rate and asymptotic hitting time of any  mixed strategy (1+1) EA  consisting of several mutation operators   is not worse  than  that of the worst pure strategy (1+1) EA using only one mutation operator. Furthermore it is proven that if these mutation operators are mutually complementary, then it is possible to design a mixed strategy (1+1) EA whose performance  is better than that of any pure strategy (1+1) EA  using only one  mutation operator.

\end{abstract}

\section{Introduction}
Different search operators have been proposed and applied in  EAs~\cite{fogel1997handbook}. Each search operator has its own advantage.   Therefore an interesting research issue is  to combine the advantages of variant operators together and then design more efficient hybrid EAs. Currently hybridization of  evolutionary algorithms becomes popular due to their capabilities in handling some real world problems~\cite{grosan2007hybrid}.

Mixed strategy EAs, inspired from strategies and games~\cite{dutta1999strategies}, aims at integrating several mutation operators into a single algorithm~\cite{he2005game}. At each generation, an individual will choose one mutation operator  according to a strategy probability distribution.  Mixed strategy evolutionary programming has been implemented for continuous optimization and experimental results show   it   performs better than its rival, i.e., pure strategy evolutionary programming  which utilizes a single mutation operator~\cite{dong2007evolutionary,liang2010mixed}.

However no analysis has been made to answer the  theoretical question:  whether and when is  the performance of mixed strategy EAs  better than that of pure strategy EAs? This paper aims at providing an initial answer. In theory, many of EAs can be regarded as a matrix iteration procedure. Following  matrix iteration analysis~\cite{varga2009matrix},  the performance of EAs is measured by  the asymptotic convergence rate, i.e., the spectral radius of a probability transition sub-matrix associated with an EA. Alternatively the performance of EAs can be measured by  the asymptotic hitting time~\cite{he2011population}, which  approximatively equals the reciprocal of the asymptotic convergence rate. Then a theoretical analysis is made to compare the performance of mixed strategy and pure strategy EAs .

The rest of this paper is organized as follows.
Section 2 describes pure strategy and mixed strategy EAs. Section
3 defines asymptotic convergence rate and asymptotic hitting time. Section 4 makes a  comparison of pure strategy and mixed strategy EAs. Section 5 concludes the paper.

\section{Pure Strategy  and Mixed Strategy EAs}
Before starting a theoretical analysis of mixed strategy EAs, we first demonstrate the result of a computational experiment.

\begin{example}
Let's see an instance of the average capacity 0-1 knapsack problem~\cite{michalewicz1996genetic,he2007comparison}:  
\begin{equation}
\begin{array}{llll}
&\mbox{maximize }  \sum^{10}_{i=1}v_i b_i,  & b_i\in \{ 0,1\}, \\
&\mbox{subject to } \sum^{10}_{i=1} w_i b_i \le C,  
\end{array}
\end{equation}
 where 
  $v_1=10$ and $v_i=1$ for $i=2, \cdots, 10$;   $w_1=9$
and $w_i=1$ for $i=2, \cdots, 10$; $C=9$.

The fitness function is that for $x=(b_1, \cdots, b_{10})$
$$
f(x)=\left \{ 
\begin{array}{llll}
&\sum^{10}_{i=1}v_i b_i, &\mbox{if } \sum^{10}_{i=1} w_i b_i \le C,  \\
& 0, &\mbox{if } \sum^{10}_{i=1} w_i b_i > C.  
\end{array}
\right.
$$

We consider two types of mutation operators:
\begin{itemize}
\item  s1:   flip each bit $b_i$ with a probability $0.1$;
\item  s2:   flip each bit $b_i$ with a probability $0.9$;
\end{itemize}

The selection operator is to accept a better offspring only.

Three (1+1) EAs are compared in the computation experiment: (1) EA(s1) which adopts   s1 only, (2)    EA(s2) with   s2 only, and (3)   EA(s1,s2) which chooses either s1 or s2  with a probability $0.5$ at each generation.

Each of these three EAs runs  100 times independently. The  computational experiment shows that EA(s1, s2)  always finds the optimal solution more quickly than other twos.
\end{example}

This is a simple case study that shows  a mixed strategy EA performs better than a pure strategy EA. In general, we need to answer the following theoretical question: whether or when do a mixed strategy EAs are better than pure strategy EAs?

Consider an instance of the discrete optimization problem which is to maximize  an objective function $f(x)$:
\begin{equation}\label{equ1}
  \max \{ f(x); x \in  S \},
\end{equation}
where $S$ a finite set. 
For the analysis convenience, suppose that all constraints have been removed through an appropriate penalty function method. Under this scenario,  all points in $S$ are viewed as feasible solutions.  In evolutionary computation, $f(x)$ is called a \emph{fitness function}.

The following notation is used in the algorithm and text thereafter.
\begin{itemize}
  \item  $x,y,z \in S$  are called \emph{points} in $S$, or  \emph{individuals} in EAs or \emph{states} in Markov chains.

   \item The  \emph{optimal set} $S_{\opt}\subseteq S$ is the set consisting of all   optimal solutions to Problem (\ref{equ1}) and    \emph{non-optimal set} $S_{\non} := S \setminus S_{\opt}$.

  \item   $t$ is the  generation counter.  A random variable $\Phi_t$  
represents the state of the $t$-th generation parent; $\Phi_{t+1/2}$ the  state of the child which is generated through mutation.
\end{itemize}

The mutation and selection operators are defined as follows:
\begin{itemize}
\item A  \emph{mutation  operator} is a probability transition from $S$ to $S$. It is defined by a \emph{mutation probability transition matrix} $\mathbf{P}_m$ whose entries are given by
\begin{equation}
 P_m(x,y), \quad x, y \in S.
\end{equation}

\item  A \emph{strict  elitist selection operator} is a mapping from $S \times S$ to $S$, that is for $x \in S$ and $y \in S$,
\begin{equation}
    z=\left\{
    \begin{array}{llll}
    x, &\mbox{if } f(y) \le f(x),\\
    y, &\mbox{if } f(y) > f(x).
    \end{array}
    \right.
\end{equation}
\end{itemize}

A \emph{pure strategy}  (1+1) EA,   which utilizes only one  mutation operator,   is described in Algorithm~\ref{alg1}.
\begin{algorithm}[ht]
\caption{Pure Strategy Evolutionary Algorithm  EA(s) } \label{alg1}
\begin{algorithmic}[1]
\STATE \textbf{input}: fitness function;
\STATE generation counter $ t\leftarrow 0$;
\STATE initialize $ \Phi_0$;
\WHILE{stopping criterion is not satisfied}
\STATE $\Phi_{t+1/2}\leftarrow$ mutate $\Phi_t$ by   mutation operator  s;
\STATE evaluate the fitness of  $\Phi_{t+1/2}$;
\STATE $\Phi_{t+1}\leftarrow$ select  one individual from $\{ \Phi_t, \Phi_{t+1/2}\}$ by strict elitist selection;
\STATE $t\leftarrow t+1$;
\ENDWHILE \STATE \textbf{output}:  the maximal value of the fitness function.
\end{algorithmic}
\end{algorithm}

The stopping criterion is that the running  stops once an optimal solution is found. If an EA cannot find an optimal solution, then it will not stop and the running time is infinite. This is common in the theoretical analysis of EAs.

Let s1, ..., s$\kappa$ be $\kappa$ mutation operators  (called  \emph{strategies}).
Algorithm~\ref{alg2} describes the procedure of a \emph{mixed strategy} (1+1) EA. At the $t$-th generation,
 one mutation operator is chosen from the $\kappa$ strategies   according to a \emph{strategy probability distribution}
\begin{equation}
 q_{s1}(x), \cdots, q_{s\kappa}(x),
\end{equation}
subject to $0\le q_s(x) \le 1$ and $\sum_s  q_{s}(x)=1$.

Write this probability distribution  in short by a vector $\mathbf{q}(x)=[q_s(x)]$.
\begin{algorithm}[ht]
\caption{Mixed Strategy Evolutionary Algorithm  EA(s1, ..., s$\kappa$)} \label{alg2}
\begin{algorithmic}[1]
\STATE \textbf{input}: fitness function;
\STATE generation counter $ t\leftarrow 0$;
\STATE initialize $ \Phi_0$;
\WHILE{stopping criterion is not satisfied}
\STATE choose a mutation operator  sk  from  s1, ..., s$\kappa$;
\STATE $\Phi_{t+1/2}\leftarrow$ mutate  $\Phi_t$ by mutation operator $sk$;
\STATE evaluate $\Phi_{t+1/2}$;
\STATE $\Phi_{t+1}\leftarrow$ select one individual from $\{\Phi_t, \Phi_{t+1/2}\}$ by  strict elitist selection;
\STATE $t\leftarrow t+1$;
\ENDWHILE \STATE \textbf{output}:  the maximal value of the fitness function.
\end{algorithmic}
\end{algorithm}

Pure strategy EAs can be regarded a special case of mixed strategy EAs with only one strategy.

EAs can be classified into two types:
\begin{itemize}
\item A \emph{homogeneous EA}  is an EA which applies the same mutation operators  
and same strategy probability distribution for all generations.

\item An \emph{inhomogeneous} EA is an EA  which doesn't apply the same mutation operators  
or same strategy probability distribution for all generations.
\end{itemize}

This paper will only discuss \emph{homogeneous EAs} mainly due to the following  reason:
\begin{itemize}
  \item The probability transition matrices of an inhomogeneous EA   may be chosen to be totally different at different generations. This makes the theoretical analysis of an inhomogeneous EA extremely hard.
\end{itemize}

\section{Asymptotic Convergence Rate and Asymptotic Hitting Time}
Suppose that a homogeneous EA is applied to maximize a fitness function $f(x)$, then the population sequence $\{\Phi_t, t=0, 1,\cdots\}$  can be modelled by a  \emph{homogeneous Markov chain} \cite{rudolph1994convergence,he1999convergence}.  Let $\mathbf{P}$ be the probability  transition matrix, whose   entries are given by
$$P(x,y)=P(\Phi_{t+1}= y \mid  \Phi_t = x), \quad x , y \in S.$$

Starting from an initial state $x$, the mean number  $m(x)$ of  generations to  find an optimal solution is called  the  \emph{hitting time} to the set $S_{\opt}$ \cite{he2003towards}.
$$\begin{array}{llll}
\tau(x)&:=&\min \{t;  \Phi_t  \in S_{\mathrm{opt}} \mid \Phi_0= x\},\\
 m (x)&:=& E[\tau(x)]=\displaystyle \sum^{+\infty}_{t=0} t P(\tau(x)=t).
\end{array}
$$

Let's arrange  all individuals in the order of  their  fitness from high to low: $ x_1, x_2, \cdots   $, then their hitting times are:
$$  m(x_1), m(x_2), \cdots   .$$
Denote it in short by a vector $\mathbf{m}=[m (x)]$.

Write the transition matrix  $\mathbf{P}$ in the canonical form  \cite{iosifescu1980finite},
\begin{equation}
\label{equ2} \mathbf{P} =
\begin{pmatrix}
\mathbf{I} & \mathbf{  0} \\
  *  & \mathbf{T}
\end{pmatrix},
\end{equation}
where $\mathbf{I}$ is a unit matrix and $\mathbf{{  0}}$ a zero matrix. $ \mathbf{T}$ denotes  the  probability transition sub-matrix among non-optimal  states, whose entries are given by
$$
 P(x,y), \quad x \in S_{\non}, y \in S_{\non}.
$$
The part $*$ plays no role in the analysis.

Since $\forall x \in S_{\opt}, m (x)=0$, it is sufficient to consider $ {m}(x)$ on non-optimal states $x \in S_{\non}$.
For the simplicity of notation,  the vector ${\mathbf{m}}$ will also denote the hitting  times for all non-optimal states:
$
 [m (x)], x \in S_{\non}.
 $

The Markov chain associated with an EA can be viewed as a matrix iterative procedure, where the iterative matrix is the probability transition sub-matrix $\mathbf{T}$.
Let $\mathbf{p}_0$ be the vector $[p_0(x)]$ which represents the probability distribution of the initial individual:
$$
p_0(x): =P(\Phi_0=x), \quad x \in S_{\non},
$$
and
$\mathbf{p}_t$  the vector $[p_t(x)]$ which represents the probability distribution of the $t$-generation individual:
$$
p_t(x): =P(\Phi_t=x), \quad x \in S_{\non}.
$$

If the spectral radius $\rho(\mathbf{T})$ of the matrix $\mathbf{T}$ satisfies: $\rho(\mathbf{T})< 1$, then we know~\cite{varga2009matrix}
$$
\lim_{t \to \infty} \parallel \mathbf{p}_{t}  \parallel  =0.
$$

Following matrix iterative analysis~\cite{varga2009matrix}, the asymptotic convergence rate of an EA is  defined as below.
\begin{definition}
The  \emph{asymptotic convergence rate} of   an EA for maximizing $f(x)$ is
\begin{equation}
R(\mathbf{T}):=-\ln \rho(\mathbf{T})
\end{equation}
where $\mathbf{T}$ is the probability transition sub-matrix restricted to non-optimal states and  $\rho(\mathbf{T})$  its spectral radius.
\end{definition}

Asymptotic convergence rate is different from previous definitions of convergence rate based on  matrix norms or probability distribution \cite{he1999convergence}.

Note:  Asymptotic convergence rate depends on both the probability transition sub-matrix $\mathbf{T}$ and  fitness function $f(x)$. Because the spectral radius of the probability transition matrix 
$\rho(\mathbf{P})=1$, thus $\rho(\mathbf{P})$ cannot be used to measure the performance of EAs.  Becaue the mutation probability transition matrix is the same for all functions $f(x)$, and   $\rho(\mathbf{P}_m)=1$, so  $\rho(\mathbf{P}_m)$ cannot be used to measure the performance of EAs too.

If $\rho(\mathbf{T})<1$, then the hitting time vector satisfies (see Theorem 3.2 in  \cite{iosifescu1980finite}),
\begin{equation}\label{equ3}
\mathbf{m}=(\mathbf{I}- \mathbf{T})^{-1} \mathbf{1}.
\end{equation}

The matrix
 $\mathbf{N}:=(\mathbf{I}-\mathbf{T})^{-1}
$ is called the \emph{fundamental matrix} of the  Markov chain, where  $\mathbf{T}$ is the probability transition sub-matrix restricted to non-optimal states.

The spectral radius $\rho(\mathbf{N})$ of  the fundamental matrix  can be used to measure   the performance  of EAs too.

\begin{definition}
The \emph{asymptotic hitting time} of an EA for maximizing $f(x)$ is
$$
T(\mathbf{T})=
\left\{ \begin{array}{lll}
\rho(\mathbf{N})=\rho((\mathbf{I}-\mathbf{T})^{-1}), &\mbox{if } \rho(\mathbf{T})<1,\\
+\infty, &\mbox{if } \rho(\mathbf{T})=1.
\end{array} 
\right. 
$$
where $\mathbf{T}$ is the probability transition sub-matrix restricted to non-optimal states and $\mathbf{N}$ is the fundamental matrix.
\end{definition}

From Lemma 5  in \cite{he2011population},, we know the asymptotic hitting time  is  between the best  and  worst case hitting times, i.e.,
\begin{equation}
   \min\{ m(x); x \in S_{\non}\}  \le   T(\mathbf{T}) \le \max \{ m(x); x \in S_{\non} \}.
\end{equation}

From Lemma 3   in \cite{he2011population}, we know
\begin{lemma} \label{lem1}
For any homogeneous (1+1)-EA using strictly elitist selection, it holds
$$
\begin{array}{llll}
&\rho(\mathbf{T})= \max\{ P(x,x); x \in S_{\non}\},\\
    &\rho(\mathbf{N})=\displaystyle \frac{1}{1-\rho(\mathbf{T})}, &\mbox{if }\rho(\mathbf{T}) <1.
\end{array}
$$
\end{lemma}

From Lemma~\ref{lem1} and Taylor series, we get that 
$$ 
R(\mathbf{T}) T(\mathbf{T}) = 
\sum^{\infty}_{k=1} \frac{1}{k } \left( \frac{1}{T(\mathbf{T})}\right)^{k-1}.
$$ 

If we make a mild assumption   $ T(\mathbf{T}) \ge 2,$ (i.e., the asymptotic hitting time  is at least two generations), then the asymptotic hitting time approximatively equals the reciprocal of the asymptotic convergence rate (see Figure~1).

\begin{figure}
\begin{center}
\begin{tikzpicture}[line cap=round,line join=round,>=triangle 45,x=5cm,y=1cm]
\draw[->,color=black] (-0.3,0) -- (1.1,0);
\foreach \x in {-0.2,0.2,0.4,0.6,0.8,1}
\draw[shift={(\x,0)},color=black] (0pt,2pt) -- (0pt,-2pt) node[below] {\footnotesize $\x$};
\draw[color=black] (0.95,0.03) node [anchor=south west] { $\rho(\mathbf{T})$};
\draw[->,color=black] (0,-0.2) -- (0,2.9);
\foreach \y in {,0.5,1,1.5,2,2.5}
\draw[shift={(0,\y)},color=black] (2pt,0pt) -- (-2pt,0pt) node[left] {\footnotesize $\y$};
\draw[color=black] (0.01,2.75) node [anchor=west] { $R (\mathbf{T})\times T(\mathbf{T}) $};
\draw[color=black] (0pt,-10pt) node[right] {\footnotesize $0$};
\clip(-0.3,-0.2) rectangle (1.1,2.9);
\draw[smooth,samples=100,domain=0.1:0.99] plot(\x,{(-ln((\x)))/(1-(\x))});
\end{tikzpicture}
\end{center}
\label{fig1}
\caption{The relationship between the asymptotic hitting time and asymptotic convergence rate: $1/R(\mathbf{T}) < T(\mathbf{T}) <1.5/R(\mathbf{T})$ if   $\rho(\mathbf{T})\ge 0.5$.}
\end{figure}
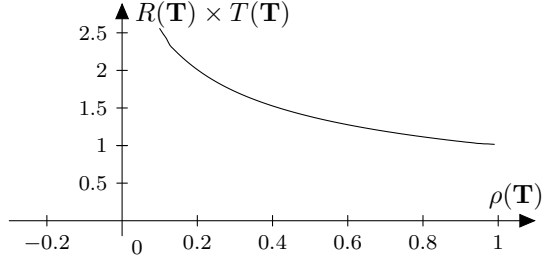

\begin{example}\label{exa2}
Consider the problem of maximizing the One-Max function: 
$$
 f(x)=  \mid x \mid,
$$
where
$x=(b_1 \cdots b_n)$ a binary string, $n$  the string length and $\mid x \mid := \sum^n_{i=1} b_i$.
The mutation operator used in the (1+1) EA is to
   choose one bit randomly and then  flip it.

Then asymptotic convergence rate and asymptotic hitting time are
$$
\begin{array}{ccc}
 {1}/{n} < R(\mathbf{T}) < {1}/{(n-1)},  \\
 T(\mathbf{T})=n.
\end{array}
$$
\end{example}

\section{A Comparison of Pure Strategy and Mixed Strategy}
In this section,  subscripts ${\mathbf{q}}$ and s are added to distinguish between a mixed strategy EA using a strategy   probability distribution $\mathbf{q}$  and a pure strategy EA using a pure strategy  s. For example, $\mathbf{T}_{\mathbf{q}}$ denotes the probability transition sub-matrix  of  a mixed strategy EA; $\mathbf{T}_{s}$ the transition sub-matrix   of a pure strategy EA.
 
\begin{theorem}
\label{the2}
Let $s1, \cdots s\kappa$ be $\kappa$ mutation operators.
\begin{enumerate}
  \item The asymptotic convergence rate  of any mixed strategy EA  consisting of   these $\kappa$ mutation operators is not smaller than the worst pure strategy  EA using only one of these  mutation operator;
   \item   and the asymptotic hitting time of any mixed strategy EA is not larger than the worst pure strategy  EA using one only of these mutation operator.
\end{enumerate}
\end{theorem}
\begin{IEEEproof}
(1) From Lemma~\ref{lem1} we know
$$ 
\rho(\mathbf{T}_{\mathbf{q}}) =\max\{ \displaystyle\frac{1}{\kappa} \sum^\kappa_{k=1}  P_{s_k} (x, x); x \in S_{\non} \} \le   \displaystyle  \frac{1}{\kappa} \sum^\kappa_{k=1}    \rho(\mathbf{T}_{sk})   
 \le     \max\{  \rho(\mathbf{T}_{sk}
); k=1, \cdots, \kappa\}.
 $$
Thus we get that
$$
R(\mathbf{T}_{\mathbf{q}}) :=-\ln  \rho (\mathbf{T}_{\mathbf{q}})  \ge   \max\{-\ln \rho(\mathbf{T}_{sk}); k=1, \cdots, \kappa\}.  
$$

 (2)
 From Lemma \ref{lem1}, we know
 $$
 \rho(\mathbf{N})=\displaystyle \frac{1}{1-\rho(\mathbf{T})},
 $$ then we get
 $
\rho(\mathbf{N}_{\mathbf{q}})
\le \max \{\rho(\mathbf{N}_{s_k}); k=1, \cdots, \kappa \}.
 $
\end{IEEEproof}

In the following we investigate whether and when the performance of a  mixed strategy EA is better than a  pure strategy EA.

\begin{definition}
A mutation operator s1
  is called  \emph{complementary} to another mutation operator s2 on a fitness function $f(x)$ if
for  any $x$ such that
\begin{equation}
 P_{s1} (x, x) =\rho(\mathbf{T}_{s1}),
\end{equation}
 it holds  
\begin{equation}
 P_{s2} (x, x)< \rho(\mathbf{T}_{s1}).
\end{equation}
\end{definition}

\begin{theorem}\label{the3}
Let  $f(x)$ be a fitness function  and EA(s1) a pure strategy EA. If a  mutation operator s2 is complementary to s1, then it is possible to design a mixed strategy EA(s1,s2) which satisfies
\begin{enumerate}
  \item its asymptotic convergence rate is larger than that of EA(s1);
  \item and its asymptotic hitting time   is shorter than that of EA(s1).
\end{enumerate}
\end{theorem}

\begin{IEEEproof}
(1) Design a mixed strategy  EA(s1,  s2)  as follows. For  any $x$ such that
 $$
 P_{s1} (x, x) =\rho(\mathbf{T}_{s1}),
 $$
 let the strategy probability distribution satisfy
   $$
   q_{s2}(x)=1.
   $$

For any other $x$, let the strategy probability distribution satisfy
   $$
   q_{s1}(x)=1.
   $$

Because s2 is complementary to s1, we get that
\[
\rho(\mathbf{T}_{\mathbf{q}}) <  \rho(\mathbf{T}_{s1}) ,
\]
and then
\[
 -\ln \rho(\mathbf{T}_{\mathbf{q}}) >  - \ln \rho(\mathbf{T}_{s1}) ,
\]
which proves the first conclusion in the theorem.

(2) From Lemma~\ref{lem1}
 $$
\rho(\mathbf{N}) = \frac{1}{1-\rho{(\mathbf{T}})}
$$
we get that
$$
\rho(\mathbf{N}_{\mathbf{q}})  < \rho(\mathbf{N}_{sk}), \quad \forall k=1, \cdots, \kappa,
$$
which proves the second conclusion in the theorem. 
\end{IEEEproof}

\begin{definition}
 $\kappa$ mutation operators $s1, \cdots, s\kappa$ are called \emph{mutually complementary}  on a fitness function $f(x)$ if
 for any $  x \in S_{\non}$ and $sl \in \{ s1, \cdots, s\kappa \} $ such that
\begin{equation}
  P_{sl} (x,x) \ge  \min \{\rho(\mathbf{T}_{s1}), \cdots, \rho(\mathbf{T}_{s\kappa}) \},  
\end{equation}
 it holds:   $\exists sk \neq sl$,
\begin{equation}
      P_{sk} (x,x) < \min \{\rho(\mathbf{T}_{s1}), \cdots, \rho(\mathbf{T}_{s\kappa}) \}.
\end{equation}
\end{definition}

\begin{theorem}\label{the4}
Let $f(x)$ be  a fitness function  and $s1, \cdots, s\kappa$ be $\kappa$ mutation operators. If these mutation operators are mutually complementary, then it is possible to design a mixed strategy EA which satisfies
\begin{enumerate}
  \item its asymptotic convergence rate is larger than that of any pure strategy  EA using one  mutation operator;
  \item and its asymptotic hitting time   is shorter than that of any pure strategy  EA using one  mutation operator.
\end{enumerate}
\end{theorem}

\begin{IEEEproof}
(1) We design a mixed strategy  EA(s1, ..., s$\kappa$)  as follows.
For any  $x$ and any strategy $sl \in \{s1, \cdots, s\kappa \}$ such that
$$
  \begin{array}{llll}
  P_{sl} (x,x)\ge \min \{\rho(\mathbf{T}_{s1}), \cdots, \rho(\mathbf{T}_{s\kappa}) \},\\
  \end{array}
  $$
from the mutually complementary condition, we know $\exists sk \neq sl$, it holds
   $$
      P_{sk} (x,x) < \min \{\rho(\mathbf{T}_{s1}), \cdots, \rho(\mathbf{T}_{s\kappa}) \}.
   $$
   Let the strategy probability distribution satisfy
   $$
   q_{sk}(x)=1.
   $$

  For any other $x$,
we assign a strategy probability distribution in any way.

Because the mutation operators are mutually complementary, we get that
\[
\rho(\mathbf{T}_{\mathbf{q}}) < \min \{\rho(\mathbf{T}_{s1}), \cdots, \rho(\mathbf{T}_{s\kappa}) \},
\]
and then
\[
 -\ln \rho(\mathbf{T}_{\mathbf{q}}) > \min \{-\ln \rho(\mathbf{T}_{s1}), \cdots, -\ln \rho(\mathbf{T}_{s\kappa}) \},
\]
which proves the first conclusion in the theorem.

(2) From Lemma~\ref{lem1}
 $$
\rho(\mathbf{N}) = \frac{1}{1-\rho{(\mathbf{T}})},
$$
we get that
$$
\rho (\mathbf{N}_{\mathbf{q}})  < \rho (\mathbf{N}_{sk}), \quad \forall k=1, \cdots, \kappa,
$$
which proves the second conclusion in the theorem. 
\end{IEEEproof}

\begin{example}\label{exa3}
Consider the problem of maximizing the following fitness function $f(x)$ (see Figure~\ref{fig2}): 
$$
f(x)=\left
\{\begin{array}{lll}
\mid x \mid, & \mbox{if } \mid x \mid < 0.5 n \mbox{ and } \mid x \mid \mbox{ is even};\\
\mid x \mid+2, & \mbox{if } \mid x \mid < 0.5 n \mbox{ and } \mid x \mid \mbox{ is odd};\\
\mid x \mid, & \mbox{if } \mid x \mid \ge 0.5 n.
\end{array}
\right.
$$
where
$x=(b_1 \cdots b_n)$ is a binary string, $n$ the string length and $\mid x \mid := \sum^n_{i=1} b_i$. 
\begin{figure}[ht]
\begin{center}
\definecolor{qqqqff}{rgb}{0,0,1}
\begin{tikzpicture}[line cap=round,line join=round,>=triangle 45,x=0.3cm,y=0.175cm]
\draw[->,color=black] (-1.5,0) -- (18.5,0);
\foreach \x in {,2,4,6,8,10,12,14,16,18}
\draw[shift={(\x,0)},color=black] (0pt,2pt) -- (0pt,-2pt) node[below] {\footnotesize $\x$};
\draw[color=black] (15.12,0.2) node [anchor=south west] { $\mid x \mid$};
\draw[->,color=black] (0,-1.5) -- (0,18.5);
\foreach \y in {,5,10,15}
\draw[shift={(0,\y)},color=black] (2pt,0pt) -- (-2pt,0pt) node[left] {\footnotesize $\y$};
\draw[color=black] (0.2,17.45) node [anchor=west] { $f(x)$};
\draw[color=black] (0pt,-10pt) node[right] {\footnotesize $0$};
\clip(-1.5,-1.5) rectangle (18.5,18.5);
\draw (0,0)-- (1,3);
\draw (1,3)-- (2,2);
\draw (2,2)-- (3,5);
\draw (3,5)-- (4,4);
\draw (4,4)-- (5,7);
\draw (5,7)-- (6,6);
\draw (6,6)-- (7,9);
\draw (7,9)-- (8,8);
\draw (8,8)-- (9.01,9.11);
\draw (9.01,9.11)-- (10,10);
\draw (10,10)-- (11,11);
\draw (11,11)-- (12,12);
\draw (12,12)-- (13,13);
\draw (13,13)-- (14,14);
\draw (14,14)-- (15,15);
\draw (15,15)-- (16,16);
\begin{scriptsize}
\fill [color=black] (1,3) circle (1.5pt);
\fill [color=black] (3,5) circle (1.5pt);
\fill [color=black] (5,7) circle (1.5pt);
\fill [color=black] (7,9) circle (1.5pt);
\fill [color=black] (0,0) circle (1.5pt);
\fill [color=black] (2,2) circle (1.5pt);
\fill [color=black] (4,4) circle (1.5pt);
\fill [color=black] (6,6) circle (1.5pt);
\fill [color=black] (8,8) circle (1.5pt);
\fill [color=black] (10,10) circle (1.5pt);
\fill [color=black] (12,12) circle (1.5pt);
\fill [color=black] (11,11) circle (1.5pt);
\fill [color=black] (13,13) circle (1.5pt);
\fill [color=black] (14,14) circle (1.5pt);
\fill [color=black] (15,15) circle (1.5pt);
\fill [color=black] (16,16) circle (1.5pt);
\fill [color=black] (9.01,9.11) circle (1.5pt);
\end{scriptsize}
\end{tikzpicture}
\caption{The shape of the   function $f(x)$ in Example \ref{exa3} when  $n=16$.}
\label{fig2}
\end{center}
\end{figure}
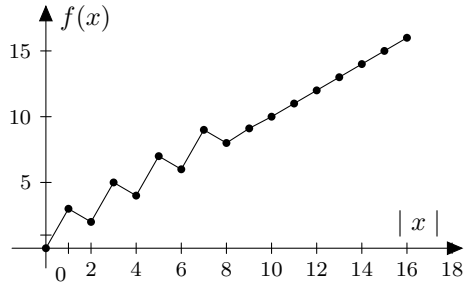

Consider two  common mutation operators:
\begin{itemize}
  \item s1: to choose one bit  randomly and then flip it;
  \item s2:  to flip each bit independently with a probability $1/n$.
\end{itemize}

EA(s1) uses the mutation operator s1 only. Then
$
\rho(\mathbf{T}_{s1})=1,
$
and then the asymptotic convergence rate is
$R(\mathbf{T}_{s1})=0.$ 

EA(s2) utilizes  the mutation operator s2 only.  Then
$$
\rho(\mathbf{T}_{s2})=1-\frac{1}{n} \left( 1- \frac{1}{n} \right)^{n-1}.
$$
We have
$$
\min \{ \rho(\mathbf{T}_{s1}),\rho(\mathbf{T}_{s2})\} = 1-\frac{1}{n} \left( 1- \frac{1}{n} \right)^{n-1}.
$$

(1) For any $x$ such that
$$
P_{s1} (x,x) \ge  1-\frac{1}{n} \left( 1- \frac{1}{n} \right)^{n-1},$$
we have 
$$
P_{s1} (x,x)  =1,$$
and we know that  
$$
P_{s2} (x,x) < 1-\frac{1}{n} \left( 1- \frac{1}{n} \right)^{n-1}.
$$

(2) For any $x$ such that
$$
P_{s2} (x,x) =\rho(\mathbf{T}_{s2})=1-\frac{1}{n} \left( 1- \frac{1}{n} \right)^{n-1},
$$
we know that 
 $$
 P_{s1} (x,x) =1-\frac{1}{n} < \rho(\mathbf{T}_{s2})=1-\frac{1}{n} \left( 1- \frac{1}{n} \right)^{n-1}.
 $$

Hence these two mutation operators are mutually complementary.

We design a mixed strategy  EA(s1,s2)  as follows:
 let the strategy probability distribution  satisfy 
  $$
  q_{s1}(x)=
  \left\{
  \begin{array}{lll}
  0, &\mbox{if } \mid x \mid \le 0.5n;\\
  1, &\mbox{if } \mid x \mid > 0.5n.
  \end{array}
  \right.
  $$

According to Theorem~\ref{the4},  the asymptotic convergence rate  of this  mixed strategy  EA(s1,s2)  is larger than that of either EA(s1) or EA(s2).
\end{example}

\section{Conclusion and Discussion}
The result  of this paper is summarized in three points.
\begin{itemize}
  \item Asymptotic convergence rate and asymptotic hitting time are proposed to measure the performance of EAs.  They are seldom used  in evaluating the performance of EAs before.

  \item It is proven  that  the asymptotic convergence rate and asymptotic hitting time of any  mixed strategy (1+1) EA  consisting of several mutation operators   is not worse  than  that of the worst pure strategy (1+1) EA using only one of these mutation operators.

  \item Furthermore, if these mutation operators are mutually complementary, then it is possible to design a mixed strategy EA whose performance (asymptotic convergence rate and asymptotic hitting time) is better than that of any pure strategy  EA using one  mutation operator.

\end{itemize}

An argument is that several mutation operators   can be applied simultaneously, e.g., in a population-based EA, different individuals  adopt different mutation operators.   However  in this case, the number of fitness evaluations  at each generation is larger than that of a (1+1) EA. Therefore  a fair comparison should be a  population-based mixed strategy EA against a population-based pure strategy EA. Due to the length restriction,   this issue will not be discussed in the paper.

\paragraph*{Acknowledgement}  
J. He is partially supported by  the EPSRC under Grant EP/I009809/1.
H. Dong is partially supported by the National Natural Science Foundation of China under Grant No.~60973075 and Natural Science Foundation of Heilongjiang Province of China under Grant No.~F200937, China.

\end{document}